\documentclass[twoside]{article}

\usepackage{aistats2025}
\usepackage{tikz}
\usepackage{multirow}
\usepackage{subcaption}
\usepackage{algorithm2e}
\usepackage{amsfonts}
\usepackage{amsthm}

\newtheorem{theorem}{Theorem}

\begin{document}

%

%

\twocolumn[

\aistatstitle{Learning Morphisms with Gauss-Newton Approximation for Growing Networks}

\aistatsauthor{ Neal Lawton \And Aram Galstyan \And  Greg Ver Steeg }

\aistatsaddress{ University of Southern California / Information Sciences Institute } ]

\begin{abstract}
A popular method for Neural Architecture Search (NAS) is based on growing networks via small local changes to the network’s architecture called network morphisms. These methods start with a small seed network and progressively grow the network by adding new neurons in an automated way. However, it remains a challenge to efficiently determine which parts of the network are best to grow. Here we propose a NAS method for growing a network by using a Gauss-Newton approximation of the loss function to efficiently learn and evaluate candidate network morphisms. We compare our method with state of the art NAS methods for CIFAR-10 and CIFAR-100 classification tasks, and conclude our method learns similar quality or better architectures at a smaller computational cost.
\end{abstract}

\section{INTRODUCTION}
Neural Architecture Search (NAS), which seeks to automate the architectural design of neural networks, has become a central problem in machine learning research~\cite{Elsken-survey}. Researchers often advance state-of-the-art by carefully designing novel network architectures for specific problems, e.g., ResNets~\cite{he2016deep} for image classification and transformers~\cite{vaswani} for natural language processing. Automating this design process can lead to significant efficiency gains. In fact, NAS is already being used to improve state-of-the-art. For example, EfficientNet \cite{tan2019efficientnet} achieves various state-of-the-art results by expanding a good baseline network found using the MnasNet architecture search method \cite{tan2019mnasnet}. Another example is the latest generation of MobileNets \cite{howard2019searching}, which was designed with assistance from NAS methods.

There are many different methods for performing NAS, such as evolutionary methods~\cite{elsken2017simple,  real2019regularized, nekrasov2019fast}, reinforcement learning methods~\cite{zoph2018learning, gong2019autogan, zhong2018practical}, pruning methods\cite{frankle2018lottery, liu2017learning, han2015deep}, and growing methods\cite{liu2018progressive, gordon2018morphnet,liu2019splitting, wu2020firefly}. The method proposed in this paper belongs to the family of growing methods. Growing methods have the benefit of being relatively computationally inexpensive: like pruning methods, they only ever train one network over the duration of the architecture search process, as opposed to evolutionary and reinforcement learning methods that train many independent networks. Further, growing methods are less computationally expensive than pruning methods, since it is less computationally expensive to train and grow from a small seed network than it is to train and prune a large network.

Growing methods also come with challenges. Growing methods grow networks via small local changes to the network's architecture called \textit{network morphisms}. These morphisms are parameterized by parameters $\theta$, so that when $\theta=0$, the input-ouput mapping of the neural network is unchanged (see Figure \ref{fig:morphism-diagrams}). To grow a network, we must choose which morphisms to apply as well as the parameters for those morphisms. Ideally, we would choose morphism parameters to minimize the loss achieved by the expanded network that results from applying the morphism, but it is computationally prohibitive to find those parameters via optimization when there are many possible morphisms to consider.

In this paper, we propose a method for learning and evaluating morphism parameters quickly and efficiently. Our method utilizes a Gauss-Newton approximation of the loss function to estimate the expected decrease in the loss that would be achieved by applying each morphism. We then optimize this approximate loss function using backpropagation to learn and evaluate morphism parameters without ever constructing a large expanded network. We use this method to design a NAS algorithm that concurrently grows and trains a network on a dynamic schedule.

We compare our method with other NAS methods on CIFAR-10 and CIFAR-100 classification tasks \cite{krizhevsky2009learning}. We observe that our method grows networks with similar or better parameter-accuracy tradeoff compared to baseline methods. Our experiments indicate that our proposed NAS method achieves state-of-the-art performance at a fraction of the computational cost. 

\section{RELATED WORK}
There are many different ways to approach neural architecture search. Evolutionary methods~\cite{elsken2017simple, real2017large, real2019regularized, nekrasov2019fast, wong2018transfer} explore the architecture search space by generating several child networks from a parent network, training those child networks, then discarding child networks that do not perform well. Reinforcement learning methods~\cite{tan2019mnasnet, zoph2018learning, gong2019autogan, zhong2018practical} view network architectures as sequences of tokens, and train a controller RNN to generate good network architectures. These reinforcement learning methods require sampling and training several hundred networks from the controller RNN. Reinforcement learning methods have been used to great success as components of larger machine learning pipelines \cite{howard2019searching, tan2019efficientnet}, and have learned networks that outperform many important hand-designed networks with unique architecture elements \cite{he2016deep, hu2018squeeze, han2017deep, zhang2018shufflenet, iandola2016squeezenet, chollet2017xception}. Evolutionary methods and reinforcement methods are well suited for exploring certain aspects of the discrete architecture search space, and can explore a larger design space than growing and pruning methods. For example, reinforcement learning methods can learn where to place architecture elements like dropout and batchnorm layers, while growing and pruning methods typically only search for the number of convolutional layers and the number of channels to use in each convolutional layer.

Pruning methods \cite{frankle2018lottery, liu2018rethinking, liu2017learning, han2015deep} have become popular for shrinking large, high-performing networks down to much smaller networks without sacrificing test accuracy. One-shot methods \cite{pham2018efficient} are similar: they simplify NAS by constraining the search space to subgraphs of a large trained network. These methods are much less computationally expensive than reinforcement learning and evolutionary methods, but still require training a large network.

The computational inexpensiveness of growing progressively larger networks has been exploited for NAS \cite{liu2018progressive, rusu2016progressive, gordon2018morphnet} and for training fixed networks \cite{karras2017progressive}. Growing networks via network morphisms has previously been used in combination with reinforcement learning \cite{cai2018efficient} and evolutionary NAS methods \cite{elsken2018efficient}. In contrast, we use network morphisms to view NAS as a continuous optimization problem, similar to other differentiable architecture search methods \cite{luo2018neural, shin2018differentiable, liu2018darts}. Unlike Net2Net \cite{chen2015net2net}, which applies network morphisms with random parameters, we build upon a recent line of work \cite{liu2019splitting, wang2019energy, wu2020steepest, wu2020firefly} that has made progress in efficiently learning and evaluating morphisms. 

We use a Gauss-Newton approximation to estimate the decrease in the loss achieved by applying a network morphism. Bayesian optimization NAS methods \cite{jin2018auto, liu2018progressive, klein2017learning, negrinho2017deeparchitect} also try to estimate the performance of new networks without training them. However, these methods use the performance of previously seen networks to predict the performance of future unseen networks, while our predictions are made independently of previously seen networks. In fact, our technique is more similar to \cite{lecun1990optimal} in which the authors use a diagonal approximation of the Hessian to estimate the change in the loss when pruning neurons.

\section{METHOD}
\subsection{Morphisms}

\begin{figure*}[t]
    \centering
    \begin{subfigure}[t]{0.4\textwidth}
        \parbox{0.4 \textwidth}{
            \centering
            \resizebox{!}{0.7 \textwidth}{
                \begin{tikzpicture}[scale=0.8]
	
	\begin{scope}[shift={(0,0)}]
		\node [draw, rectangle] (x) at (0,  2) {$x$};
		\node [draw, circle] (y) at (0,  0) {$y$};
		\node [draw, rectangle] (z) at (0, -2) {$z$};
		\draw[->, thick] (x) -- (y) node [midway, left] {$w_{in}$};
		\draw[->, thick] (y) -- (z) node [midway, left] {$w_{out}$};
	\end{scope}

	\begin{scope}[shift={(1.0,0)}]
		\node (arrow) at (0, 0) {$\implies$};
	\end{scope}

	\begin{scope}[shift={(3,0)}]
		\node [draw, rectangle] (x) at (0,  2) {$x$};
		\node [draw, circle] (y1) at (-1,  0) {$y_1$};
		\node [draw, circle] (y2) at (1,  0) {$y_2$};
		\node [draw, rectangle] (z) at (0, -2) {$z$};
		\draw[->, thick] (x) -- (y1) node [midway, left] {$w_{in} + \theta$};
		\draw[->, thick] (x) -- (y2) node [midway, right] {$w_{in} - \theta$};
		\draw[->, thick] (y1) -- (z) node [midway, left] {$\frac{1}{2} w_{out}$};
		\draw[->, thick] (y2) -- (z) node [midway, right] {$\frac{1}{2} w_{out}$};
	\end{scope}

\end{tikzpicture}
            }
        }
        \caption{Channel splitting morphism}
        \label{fig:channel-splitting-morphism}
    \end{subfigure}
    \qquad
    \begin{subfigure}[t]{0.4\textwidth}
        \parbox{0.4 \textwidth}{
            \centering
            \resizebox{!}{0.7 \textwidth}{
                \begin{tikzpicture}[scale=0.8]
	
	\begin{scope}[shift={(0,0)}]
		\node [draw, rectangle] (x) at (0,  2) {$x$};
		\node [draw, circle] (y) at (0,  0) {$y$};
		\node [draw, rectangle] (z) at (0, -2) {$z$};
		\draw[->, thick] (x) -- (y) node [midway, left] {$w_{in}$};
		\draw[->, thick] (y) -- (z) node [midway, left] {$w_{out}$};
	\end{scope}

	\begin{scope}[shift={(1.5,0)}]
		\node (arrow) at (0, 0) {$\implies$};
	\end{scope}

	\begin{scope}[shift={(3,0)}]
		\node [draw, rectangle] (x) at (0,  2) {$x$};
		\node [draw, circle] (y) at (0,  0) {$y$};
		\node [draw, rectangle] (z) at (0, -2) {$z$};
		\draw[->, thick] (x) -- (y) node [midway, left] {$w_{in} - \theta$};
		\draw[->, thick] (y) -- (z) node [midway, left] {$w_{out}$};
	\end{scope}

\end{tikzpicture}
            }
        }
        \caption{Channel pruning morphism}
        \label{fig:channel-pruning-morphism}
    \end{subfigure}
\caption{Network morphisms. Square nodes represent convolutional layers, circular nodes represent convolutional channels.}
\label{fig:morphism-diagrams}
\end{figure*}
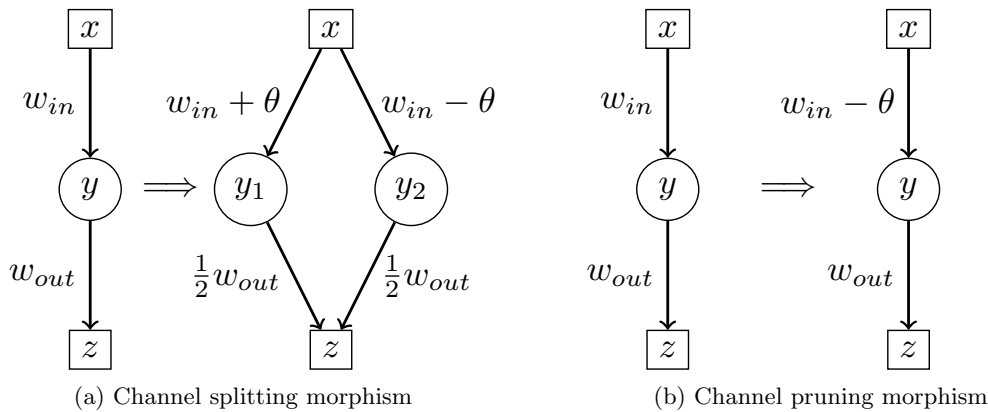

A \textit{network morphism} is a small change in a neural net's architecture parameterized by $\theta$ so that when $\theta = 0$, the morphism is function-preserving, i.e., the input-output mapping of the neural network is unchanged. In this paper, we consider several morphisms.

The first morphism we consider is a channel-splitting morphism that grows a network wider, depicted in Figure \ref{fig:channel-splitting-morphism}. For a convolutional channel $y$ with input from a layer $x$ with incoming kernel parameters $w_{in}$ and output to a layer $z$ with outgoing kernel parameters $w_{out}$, applying the channel-splitting morphism replaces the channel $y$ with two channels $y_1$ and $y_2$, with incoming kernel parameters $w_{in} + \theta$ and $w_{in} - \theta$ respectively, and each with outgoing kernel parameters $w_{out}/2$. If $\theta=0$, then this morphism duplicates the channel $y$ without changing the input-output mapping of the neural network; if $\theta \neq 0$, then this morphism replaces the feature detected by $y$ with two new feature detectors with parameters $w_{in} + \theta$ and $w_{in} - \theta$. For example, if $y$ is an edge detector, then $y_1$ and $y_2$ may detect two similar edges with slightly different angles, or along slightly different color gradients.

The second morphism we consider is a channel-pruning morphism, depicted in Figure \ref{fig:channel-pruning-morphism}. For a channel $y$ with incoming kernel parameters $w_{in}$, applying the channel-pruning morphism subtracts $\theta$ from $w_{in}$. If $\theta=0$, then the network is unchanged, but if $\theta = w_{in}$, then the incoming kernel parameters of $y$ are zero, and the channel $y$ can be pruned. The parameters of this morphism are not learned, but instead are always chosen to be $\theta = w_{in}$. 



Though not depicted in these diagrams, it is important to note that the activations of each layer are passed through a channelwise function $\sigma_\beta(\cdot)$ with different learnable parameters $\beta$ for each layer. Typically $\sigma_\beta(\cdot)$ is a composition of several simple channelwise operations, such as ReLU, dropout, batch normalization, and depthwise convolution. In this case, $\beta$ consists of the learnable parameters of the batch normalization layer and depthwise convolution (ReLU and dropout operators do not have learnable parameters). In the channel-splitting morphism, $y_1$ and $y_2$ use the same $\sigma$ and $\beta$ as the original channel $y$. 

To apply a particular morphism, we must first choose values for the morphism's parameters $\theta$. The best choice for the morphism's parameters would maximally decrease the loss of the network when the morphism is applied. Calculating this decrease exactly on a mini-batch of samples requires expanding the network by applying the morphism, then executing several forward passes through the expanded network. Optimizing the morphism parameters with SGD requires several forward and backward passes through the expanded network. This is computationally prohibitive when there are many morphisms under consideration, e.g., splitting every channel in the network.

\subsection{Gauss-Newton Approximation}

Instead, we can approximate the decrease in the loss function for each morphism. We assume that each morphism is \textit{local}, so that there exists a collection of network activations $z$ such that the mapping between the network input and any activation higher than $z$ in the computational DAG is unchanged for any choice of $\theta$. Consider the expanded networks depicted in Figure \ref{fig:morphism-diagrams}. Denote $\Delta \mathcal L(\theta)$ the change in the loss function after applying the morphism with parameters $\theta$; $\Delta z(\theta)$ the change in $z$ after applying the morphism with parameters $\theta$; $g$ the gradient of the loss function with respect to $z$ at $\theta=0$; and $H$ the Hessian of the loss function with respect to $z$ at $\theta=0$. Consider the second-order approximation of the change in loss function with respect to $z$ centered at $\theta=0$: 
\[
\Delta \mathcal L(\theta) \approx \Delta z(\theta) \cdot g + \frac{1}{2} \Delta z(\theta)^\top H \Delta z(\theta)
\]
This may be a good approximation, but computing the Hessian matrix of second derivatives is undesirable. Instead, we can make a Gauss-Newton approximation of the Hessian matrix, where $\hat {\mathcal L}$ is the current training loss:
\[
H \approx \frac{1}{2 \hat{\mathcal L}} g g^\top
\]
Note that this variation of the Gauss-Newton approximation scales correctly: if we scale the loss $\mathcal L$ by a constant factor, then the true Hessian $H$ scales by the same factor, as does our approximation.

Plugging in this approximation yields:
\begin{gather*}
\Delta \mathcal L(\theta) \approx \Delta z(\theta) \cdot g + \frac{1}{4 \hat{\mathcal L}} (\Delta z(\theta) \cdot g)^2
\end{gather*}
Note the Gauss-Newton approximation makes a rank-1 approximation of the Hessian. If the true Hessian is not low-rank, we expect the Gauss-Newton approximation to be inaccurate. It is important to consider this when approximating the Hessian for a mini-batch of samples: the joint sample Hessian for a mini-batch of samples is a block diagonal matrix with one block per sample, and is therefore probably not low-rank. However, each diagonal block of the joint sample Hessian is dense and more likely to be low-rank. For this reason, it is important to apply the Gauss-Newton approximation independently for each sample. For these reasons, for a mini-batch of samples we approximate
\begin{gather}
\Delta \mathcal L(\theta) \approx \frac{1}{S} \sum_{s=1}^S \Delta z_s(\theta) \cdot g_s + \frac{1}{4 \hat{\mathcal L}} (\Delta z_s(\theta) \cdot g_s)^2
\label{eq:gn}
\end{gather}
Recent work \cite{wu2020steepest, wang2019energy} also uses a second-order approximation of the loss function to learn morphism parameters. In that work, the authors make a second-order approximation of the loss function with respect to $\theta$. In contrast, we make a second-order approximation of the loss function with respect to $\Delta z(\theta)$. Critically, because $\Delta z(\theta)$ is a non-linear function of $\theta$, our Gauss-Newton approximation is still a high-order approximation of $\Delta \mathcal L(\theta)$ with respect to $\theta$.
\subsection{Algorithm}

Note that $g$ is the usual gradient of the loss function with respect to $z$ computed during backpropagation. In PyTorch, we use backward hooks on each layer to capture the gradient $g$ of the loss function with respect to the layer output. Then we compute $\Delta z(\theta)$ for each morphism using the current morphism parameters $\theta$ and the output from layer $x$ to execute a forward pass through the small 2-layer network on the right hand side of each diagram in Figure \ref{fig:morphism-diagrams}. 
Then we use the computed $g$ and $\Delta z(\theta)$ to construct the Gauss-Newton approximation in (\ref{eq:gn}) 
and differentiate that approximation with respect to $\theta$. After computing morphism parameter gradients this way, morphism parameters can be updated with a user-specified optimization algorithm. For a summary of this process, see Algorithm \ref{alg:algorithm}.

Since $\Delta \mathcal L(\theta)$ is computed independently for each training mini-batch, we record an exponential moving average of $\Delta \mathcal L(\theta)$ across mini-batches using a momentum hyperparameter to get a lower variance estimate of the decrease in the loss function. After computing a low-variance estimate of $\Delta \mathcal L(\theta)$, we can then weigh the tradeoff for each morphism between the estimated change in loss $\Delta \mathcal L(\theta)$ and the change in computational resource cost incurred, e.g., the number of parameters introduced, when applying the morphism. To quantify this tradeoff, we introduce a regularization hyperparameter indicating the desired tradeoff between training loss and number of parameters. Then we say a morphism has a positive loss-resource tradeoff if
\[
-\Delta \mathcal L(\theta) > \lambda_p \Delta R_p
\]
where $\lambda_p$ is the hyperparameter regularization constant on the number of parameters and $\Delta R_p$ is the change in the number of parameters resulting from applying the morphism. Given the exponential moving average estimate of $\Delta \mathcal L(\theta)$, checking whether a morphism has positive loss-resource tradeoff takes constant time. 

Although it is possible to compute model and morphism parameter gradients in the same backward pass, empirically we find that updating morphism and model parameters concurrently yields noisy morphism parameters that do not perform well. Instead, we propose a phased algorithm for growing while training, summarized in Algorithm \ref{alg:algorithm}. In this algorithm, training and growing are separated into alternating phases of length $n_{\text{phase}}$ epochs, where $n_{\text{phase}}$ is a hyperparameter. In all our experiments, we use $n_{\text{phase}}=20$. In the first phase, only model parameters are updated; in the second phase, only morphism parameters are updated. At the end of the morphism learning phase, we compute each morphism's loss-resource tradeoff. Then for each layer, we apply the top $30\%$ of morphisms local to that layer with positive loss-resource tradeoff.

Note that applying all morphisms with positive tradeoff at the end of each morphism learning phase can be unstable. This is because each computed $\Delta \mathcal L(\theta)$ estimates the change in loss function from applying a single morphism. However, these $\Delta \mathcal L(\theta)$'s are not additive: the change in loss from applying many morphisms simultaneously may not be well approximated by the sum of their $\Delta \mathcal L(\theta)$'s. In fact, we almost always observe a small increase in training loss when applying morphisms at the end of each growing phase, even when the sum of the morphisms' $\Delta \mathcal L(\theta)$'s is negative.

\begin{algorithm}[ht]
\caption{Growing networks with Gauss-Newton}
\label{alg:algorithm}
\KwData{Dataset $D$, model $M$, phase length $n_{\text{phase}}$}
\For{$t = 1, \dots, n_{\text{phase}}$}{
	\ForEach {mini-batches $\{d_s\} \in D$}{
		Compute mini-batch loss $\mathcal L = M(\{d_s\})$\;
		Compute all $\nabla_w \mathcal L$ with backprop\;
		SGD step model parameters $w$\;
	}
}
\For{$t = 1, \dots, n_{\text{phase}}$}{
	\ForEach {mini-batches $\{d_s\} \in D$}{
		  Compute all $\Delta \mathcal L(\theta)$ and $\nabla_\theta \mathcal L$ with backprop\;
    	Update exponential moving average of $\Delta \mathcal L(\theta)$\;
		SGD step morphism parameters $\theta$\;
	}
}
\ForEach{top $30\%$ morphisms with positive tradeoff}{
	Apply morphism\;
}
\end{algorithm}

\section{EXPERIMENTS}
\begin{figure*}[t]
\centering
\includegraphics[scale=0.085]{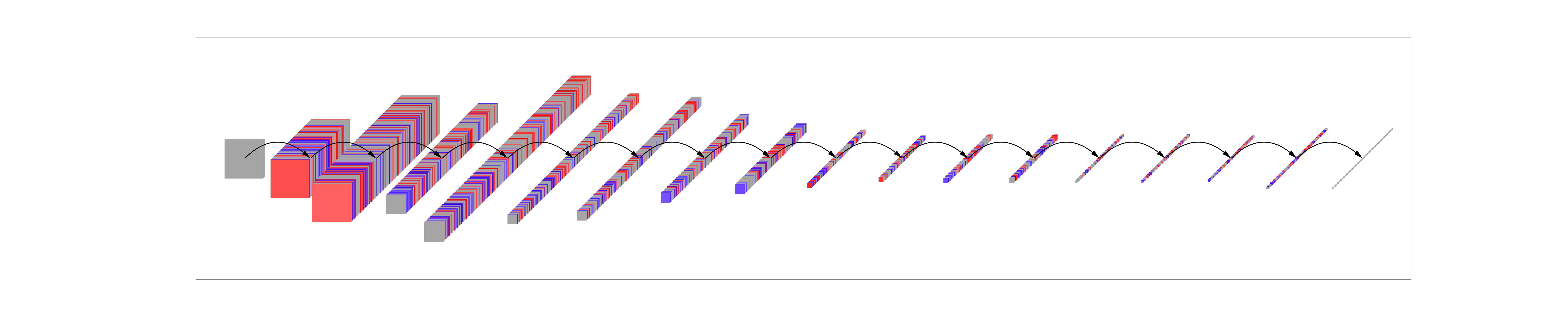}
\caption{Network grown from a VGG-19 seed network by our algorithm for classifying CIFAR-100. Here the network is at the end of its $15$-th growth phase. Channels colored red will be split with their learned channel-splitting morphism parameters in the next epoch; splitting the reddest channels is estimated to give the highest loss-resource tradeoff. Channels colred blue will be pruned in the next epoch; pruning the bluest channels is estimated to give the highest loss-resource tradeoff.}
\label{fig:example-network}
\end{figure*}
In all our experiments, we train with a batch size of $64$ and use a simple data augmentation scheme for CIFAR-10 and CIFAR-100: random horizontal flips and random crops with padding $4$.

\subsection{Gauss-Newton Approximation Accuracy}
Here we evaluate the accuracy of our Gauss-Newton approximation of the loss. We begin by constructing a VGG-19 model for CIFAR-10 and equip it with channel-splitting morphisms, one for each channel in each convolutional layer in the network. We trained the VGG-19 model with SGD with learning rate $0.1$ for $20$ epochs while holding morphism parameters constant. Then we updated morphism parameters with Adam with learning rate $10^{-2}$ for $20$ epochs while updating the exponential moving average estimate of $\Delta \mathcal L$ using momentum hyperparameter $\frac{64}{50000} \times \frac{1}{2}$ so that our estimate of $\Delta \mathcal L$ is approximately an average over the last $2$ epochs. We then computed the true change in loss achieved by each morphism with its current parameters by applying each morphism to construct an independent expanded network and evaluating that expanded network on the test dataset. We then compared our exponential moving average estimate of $\Delta \mathcal L$ with the true value. 

The results are depicted in Figure \ref{fig:gn-approx-epoch-40}. Each circle represents a single channel-splitting morphism in the specified layer. There are $64$ channels in the first layer of the VGG-19 model, and $512$ channels each in the $9$-th and last layers. The figure plots our exponential moving average estimate of $\Delta \mathcal L$ against the true $\Delta \mathcal L$ computed via brute force. If our method were $100\%$ accurate, all circles would lie on the grey dashed lines. 

The figure shows that the Gauss-Newton approximation used by our algorithm is quite accurate. This result by itself is significant. Other methods expend enormous computational resources trying to estimate how the loss of a network changes when channels are added or removed from the network. This result shows that the change in loss can be approximated to a high degree of accuracy using only statistics of the network, namely $\Delta z(\theta) \cdot g$.

We also observe that the Gauss-Newton approximation seems to be most accurate for the layer closest to the network output, and least accurate for the layer closest to the network input, though the reason for this behavior is unclear.

We conclude that the Gauss-Newton approximation used by our algorithm estimates $\Delta \mathcal L$ for each morphism to a high degree of accuracy.

\begin{figure*}[t]
\centering
\begin{subfigure}[t]{0.3 \textwidth}
    \centering
        \includegraphics[width=\textwidth]{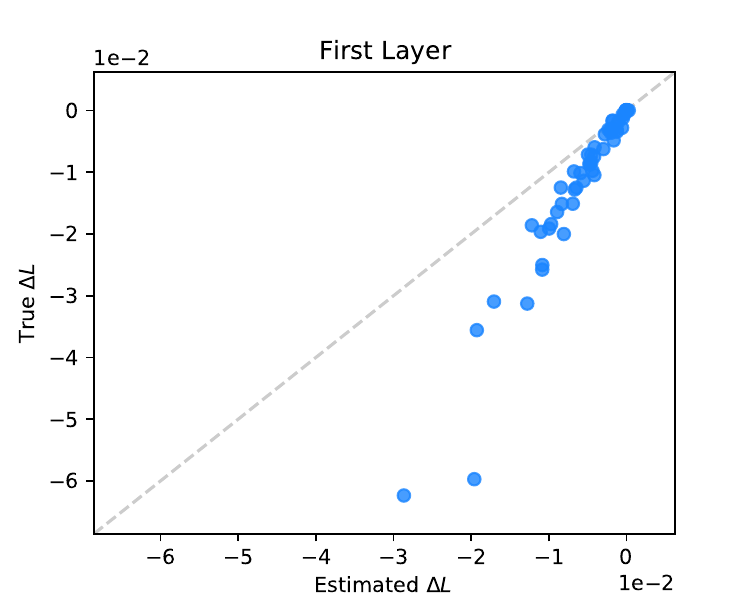}
    \label{fig:gn-approx-epoch-40-layer0}
\end{subfigure}
\hfill
\begin{subfigure}[t]{0.3 \textwidth}
    \centering
        \includegraphics[width=\textwidth]{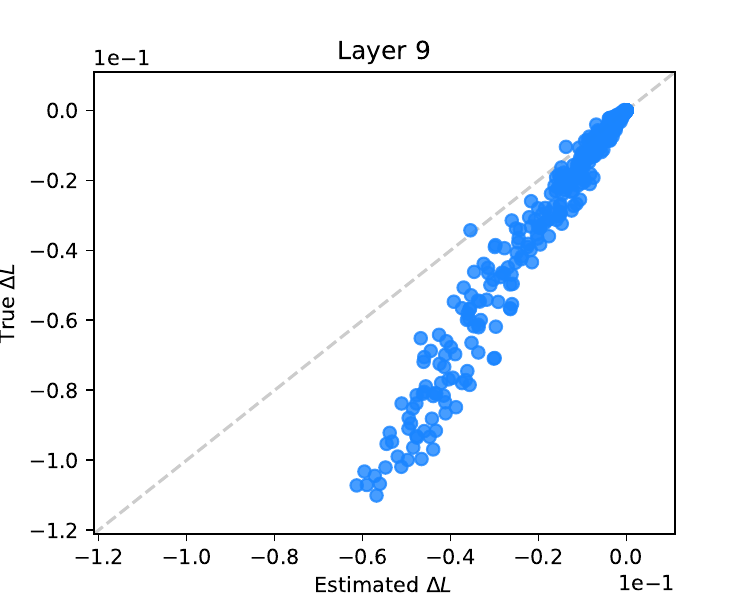}
    \label{fig:gn-approx-epoch-40-layer8}
\end{subfigure}
\hfill
\begin{subfigure}[t]{0.3 \textwidth}
    \centering
        \includegraphics[width=\textwidth]{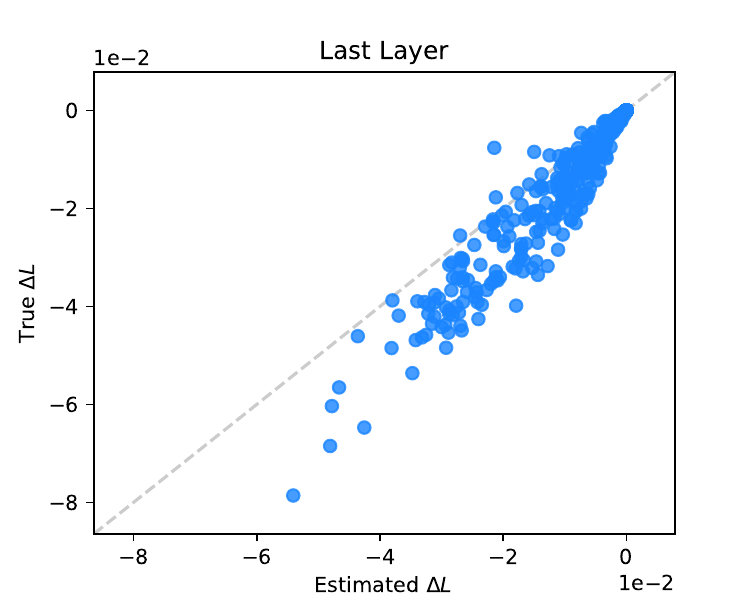}
    \label{fig:gn-approx-epoch-40-layer15}
\end{subfigure}
\caption{Estimated versus actual decrease in loss for morphisms learned while holding model parameters constant. }
\label{fig:gn-approx-epoch-40}
\end{figure*}

\subsection{Learned Morphism Quality}
Here we compare the quality of our learned morphisms to those learned via other methods. Another method for learning morphism parameters is to apply the morphism to construct an expanded network, then optimize the loss of the expanded network with respect to the morphism parameters. This allows us to learn morphism parameters that minimize the loss rather than an approximation of the loss, but is computationally expensive to scale when there are many morphisms under consideration.

Another method for choosing morphism parameters is to use the steepest descent direction as in \cite{wu2020steepest, wang2019energy, wu2020firefly}. However, the steepest descent direction does not indicate the optimal scale for $\theta$. To approximately compute the optimal scale, we perform a line search along the steepest descent direction, though this is computationally expensive.

We compare the true decrease in loss achieved by the morphism parameters learned by our algorithm with the true decrease in loss achieved by the morphism parameters produced by the two baselines described above. We do this for each of the possible $64$ channel-splitting morphisms in the first layer of the VGG-19 network trained in the previous experiment. The result is in Figure \ref{fig:compare-directions}. Each $3$-bar cluster plots the true decrease in loss achieved by the morphism parameters learned by each method for the corresponding channel-splitting morphism. For ease of viewing, we have sorted the channels with respect to the true decrease in loss achieved by the first baseline method. Note that for some channels, none of the methods are able to find good morphism parameters. After inspecting these features, we observe that at this point in training (epoch $20$), those channels have already ``died'' due to $L_2$ weight regularization, so it is likely not possible to split such bad feature detectors into two good feature detectors.

From the figure, we observe that the morphisms learned by our method most often achieve a greater decrease in loss than those learned by the steepest descent with line search baseline method. We also observe that the true decrease in loss achieved by our learned morphisms most often comes within a constant factor of the decrease achieved by the expensive network expansion baseline. We observe this most often among the morphisms with the highest potential decrease in loss; this is important, since these are the morphisms that will be selected by our algorithm to be applied to grow the network. We conclude that our algorithm learns high quality morphisms, on par with the expensive network expansion baseline method.

\begin{figure*}[t]
\centering
\includegraphics[width=\textwidth]{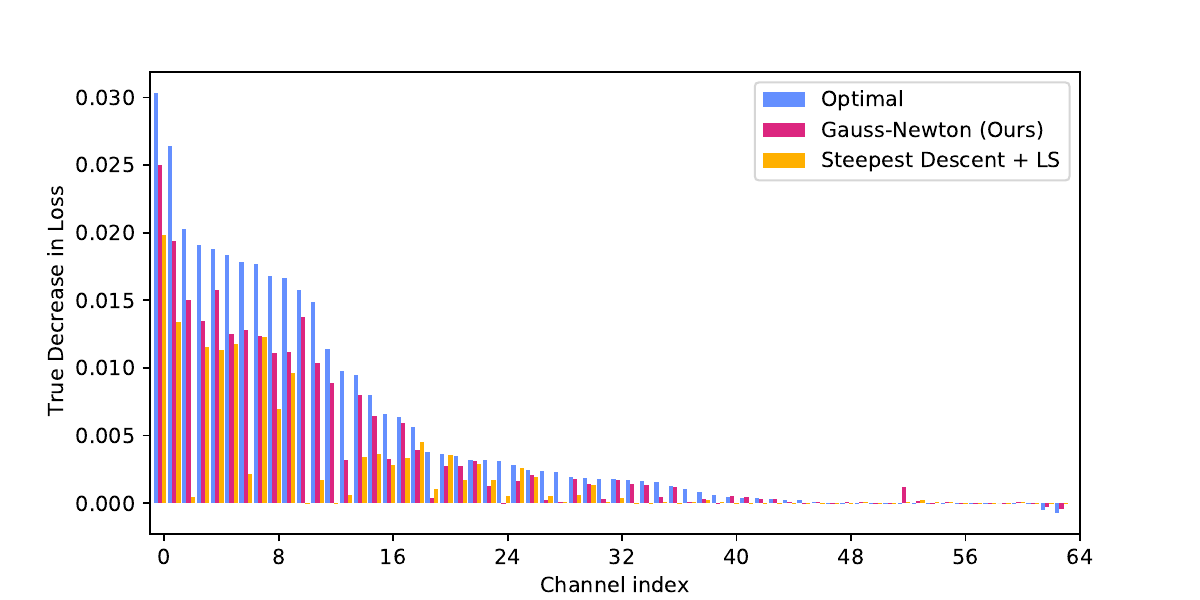}
\caption{Comparison of different morphism learning strategies. Each $3$-bar cluster plots the true decrease in loss achieved by the channel-splitting morphism learned by each method for one of the $64$ channels in the first layer of VGG-19. }
\label{fig:compare-directions}
\end{figure*}

\subsection{End-to-end Evaluation}
Here we compare our NAS algorithm end-to-end with other methods for learning architectures for classifying CIFAR-10 and CIFAR-100. We experiment with different choices of the loss-resource tradeoff hyperparameter to grow networks of many different sizes. We grow networks from one of two seed networks. The first is a VGG-19 network with $16$ channels in each convolutional layer. The second is a MobileNetV1 network with $32$ channels in each convolutional layer. In each experiment, we run our algorithm for a total of $30$ training and growing phases. We optimized model parameters using SGD with Nesterov momentum $0.9$, weight decay $10^{-4}$, and a learning rate that begins at $0.1$ and decreases by a factor of $10$ at epochs $300$ and $450$. We optimized morphism parameters with Adam and a learning rate of $10^{-2}$. After our algorithm terminates, we reinitialize the network's model parameters and retrain the model from scratch to more accurately determine the best test accuracy achievable for the learned architecture.

A visualization of a network grown by our algorithm from a VGG-19 seed network for classifying CIFAR-100 is in Figure~\ref{fig:example-network}. It is worthwhile to point out that growing from a uniform-width seed network, our algorithm naturally discovers that a unique, bottleneck-shaped architecture provides the best loss-parameter tradeoff.

Next, we report the results for CIFAR-10 and CIFAR-100 classification tasks in Tables \ref{table:cifar10} and \ref{table:cifar100}, respectively. We compare with other NAS methods as well as human-designed baselines. We observe that our method produces networks with similar or better parameter-accuracy tradeoff at a smaller computational cost. For example, a network we grew from a VGG-19 seed network using $\lambda_p = 3 \times 10^{-7}$ achieved $5.6\%$ test error on CIFAR-10 using only $1.2$ million parameters, which achieves lower test error with fewer parameters compared to \cite{liu2017learning}, which pruned a VGG-19 model down to $2.3$ million parameters and achieved $6.2\%$ test error. Similarly, a network we grew from a MobileNetV1 seed network using $\lambda_p = 3 \times 10^{-7}$ achieved $25.9\%$ test error on CIFAR-100 using only $1.4$ million parameters, which achieves lower test error with fewer parameters compared to \cite{he2016deep}, a ResNet that achieves $27.2\%$ test error with $1.7$ million parameters.

Note that our method of growing from simple VGG-19 and MobileNetV1 networks with simple channel splitting and pruning morphisms is not enough to outperform complex architectures like those produced by NASNET. Architecture elements necessary for high performance, like residual connections and squeeze-excite modules, make growing complicated because they force several layers to have the same number of channels, disallowing us from splitting channels in different layers independently. It may be possible to grow from these types of seed networks using more complex morphisms that split channels in multiple layers jointly, but this is left for future work.

\begin{table*}[ht]
\centering
\begin{tabular}{|c|c|c|c|c|c|}
\hline
Method Type & Reference & Error & Params & Reachable & GPU time \\
 &  & (\%) & (Millions) & & (days) \\
\hline \hline
\multirow{3}{*}{SOTA} 
& AmoebaNet-A \cite{real2019regularized}      & 3.3 &  3.2 &  & $3150$ \\
& NASNET-A \cite{zoph2018learning}            & 3.4 &  3.3 &  & $2000$ \\
& Large-scale Evolution \cite{real2017large}  & 5.4 &  5.4 &  & $2600$ \\
\hline \hline
\multirow{4}{*}{Morphisms} 
& NASH \cite{elsken2017simple}                                  & 5.2 & 19.7 & \checkmark & $1.0$ \\
& Slimming \cite{liu2017learning} & 6.2 & 2.3 & \checkmark & - \\& Firefly \cite{wu2020firefly}                                 & 6.2 &  1.9 & \checkmark & - \\
& Net2Net \cite{chen2015net2net}                                 & 6.5 &  3.9 & \checkmark & $2.1$ \\
\hline \hline
\multirow{4}{*}{Human-\newline Designed} 
& DenseNet \cite{huang2017densely}                              & 3.5 & 25.6 &  & N/A \\
& VGG-19 Baseline \cite{simonyan2014very}                          & 6.3 & 20.0 & \checkmark & N/A \\
& ResNet \cite{huang2016deep}                                      & 6.4 &  1.7 &  & N/A \\
& MobileNetV1 Baseline \cite{howard2017mobilenets}                 & 6.6 & 3.2 & \checkmark & N/A \\
\hline \hline
\multirow{5}{*}{Ours} 
& Seed VGG-19, $\lambda_p = 3 \times 10^{-7}$         & 5.6 & 1.2 & \checkmark & $0.7$ \\
& Seed VGG-19, $\lambda_p = 1 \times 10^{-6}$        & 6.5 & 0.6 & \checkmark & $0.5$ \\
& Seed MobileNetV1, $\lambda_p = 3 \times 10^{-8}$ & 5.8 & 0.8 & \checkmark & $1.0$ \\
& Seed MobileNetV1, $\lambda_p = 3 \times 10^{-7}$ & 6.0 & 0.5 & \checkmark & $1.0$ \\
& Seed MobileNetV1, $\lambda_p = 1 \times 10^{-6}$ & 6.2 & 0.4 & \checkmark & $0.7$ \\
\hline
\end{tabular}
\caption{Classification performance of various architectures on CIFAR-10.}
\label{table:cifar10}
\end{table*}

\begin{table*}[ht]
\centering
\begin{tabular}{|c|c|c|c|c|c|}
\hline
Method Type & Reference & Error & Params & Reachable & GPU time \\
 &  & (\%) & (Millions) & & (days) \\
\hline \hline 
\multirow{2}{*}{SOTA}
& Large-scale Evolution \cite{real2017large}  & 23.0 & 40.4 &  & - \\
& SMASH \cite{brock2017smash} & 22.1 & 4.6 &  & - \\
\hline \hline
\multirow{2}{*}{Morphisms}
& NASH \cite{elsken2017simple}  & 23.4 & 22.3 & \checkmark & $1.0$ \\
& Slimming \cite{liu2017learning} & 26.5 & 5.0 & \checkmark & - \\
\hline \hline
\multirow{4}{*}{Human \linebreak Designed}
& DenseNet \cite{huang2017densely}        & 17.2 & 25.6 & & N/A \\
& Resnet \cite{he2016deep}                & 27.2 & 1.7 & & N/A \\
& VGG-19 Baseline \cite{simonyan2014very}                         & 27.6 & 20.1 & \checkmark & N/A \\
& MobileNetV1 Baseline \cite{howard2017mobilenets}                   & 28.7 & 3.3 & \checkmark & N/A \\
\hline \hline
\multirow{5}{*}{Ours} 
& Seed VGG-19, $\lambda_p = 3 \times 10^{-7}$ & 27.2 & 2.2 & \checkmark & 0.7 \\
& Seed VGG-19, $\lambda_p = 6 \times 10^{-7}$ & 28.0 & 1.6 & \checkmark & 0.6 \\
& Seed MobileNetV1, $\lambda_p = 1 \times 10^{-6}$ & 27.2 & 0.8 & \checkmark & 0.8 \\
& Seed MobileNetV1, $\lambda_p = 6 \times 10^{-7}$ & 26.9 & 1.3 & \checkmark & 1.0 \\
& Seed MobileNetV1, $\lambda_p = 3 \times 10^{-7}$ & 25.9 & 1.4 & \checkmark & 1.0 \\
\hline
\end{tabular}
\caption{Classification performance of various architectures on CIFAR-100.}
\label{table:cifar100}
\end{table*}

\section{CONCLUSION}
In this paper, we presented a neural architecture search method for growing a network with network morphisms while training. We used a Gauss-Newton approximation of the loss to learn morphism parameters and to estimate the change in the loss resulting from applying those morphisms. We used the estimated change in loss to compute a loss-resource tradeoff for each morphism using hyperparameters that regularized the number of parameters of the grown network. We evaluated the goodness of our Gauss-Newton approximation and found that the Gauss-Newton approximation is highly accurate and yields morphism parameters that are close to optimal. We compared our method with state of the art NAS methods for classifying CIFAR-10 and CIFAR-100 and concluded that our algorithm finds similar or better architectures at a smaller computational cost.

\clearpage

\bibliography{aistats2025}

\begin{thebibliography}{10}

\bibitem{brock2017smash}
Andrew Brock, Theodore Lim, James~M Ritchie, and Nick Weston.
\newblock Smash: one-shot model architecture search through hypernetworks.
\newblock {\em arXiv preprint arXiv:1708.05344}, 2017.

\bibitem{cai2018efficient}
Han Cai, Tianyao Chen, Weinan Zhang, Yong Yu, and Jun Wang.
\newblock Efficient architecture search by network transformation.
\newblock In {\em Proceedings of the AAAI Conference on Artificial Intelligence}, volume~32, 2018.

\bibitem{chen2015net2net}
Tianqi Chen, Ian Goodfellow, and Jonathon Shlens.
\newblock Net2net: Accelerating learning via knowledge transfer.
\newblock {\em arXiv preprint arXiv:1511.05641}, 2015.

\bibitem{chollet2017xception}
Fran{\c{c}}ois Chollet.
\newblock Xception: Deep learning with depthwise separable convolutions.
\newblock In {\em Proceedings of the IEEE conference on computer vision and pattern recognition}, pages 1251--1258, 2017.

\bibitem{elsken2017simple}
Thomas Elsken, Jan-Hendrik Metzen, and Frank Hutter.
\newblock Simple and efficient architecture search for convolutional neural networks.
\newblock {\em arXiv preprint arXiv:1711.04528}, 2017.

\bibitem{elsken2018efficient}
Thomas Elsken, Jan~Hendrik Metzen, and Frank Hutter.
\newblock Efficient multi-objective neural architecture search via lamarckian evolution.
\newblock {\em arXiv preprint arXiv:1804.09081}, 2018.

\bibitem{Elsken-survey}
Thomas Elsken, Jan~Hendrik Metzen, and Frank Hutter.
\newblock Neural architecture search: A survey.
\newblock {\em Journal of Machine Learning Research}, 20(55):1--21, 2019.

\bibitem{frankle2018lottery}
Jonathan Frankle and Michael Carbin.
\newblock The lottery ticket hypothesis: Finding sparse, trainable neural networks.
\newblock {\em arXiv preprint arXiv:1803.03635}, 2018.

\bibitem{gong2019autogan}
Xinyu Gong, Shiyu Chang, Yifan Jiang, and Zhangyang Wang.
\newblock Autogan: Neural architecture search for generative adversarial networks.
\newblock In {\em Proceedings of the IEEE/CVF International Conference on Computer Vision}, pages 3224--3234, 2019.

\bibitem{gordon2018morphnet}
Ariel Gordon, Elad Eban, Ofir Nachum, Bo~Chen, Hao Wu, Tien-Ju Yang, and Edward Choi.
\newblock Morphnet: Fast \& simple resource-constrained structure learning of deep networks.
\newblock In {\em Proceedings of the IEEE conference on computer vision and pattern recognition}, pages 1586--1595, 2018.

\bibitem{han2017deep}
Dongyoon Han, Jiwhan Kim, and Junmo Kim.
\newblock Deep pyramidal residual networks.
\newblock In {\em Proceedings of the IEEE conference on computer vision and pattern recognition}, pages 5927--5935, 2017.

\bibitem{han2015deep}
Song Han, Huizi Mao, and William~J Dally.
\newblock Deep compression: Compressing deep neural networks with pruning, trained quantization and huffman coding.
\newblock {\em arXiv preprint arXiv:1510.00149}, 2015.

\bibitem{he2016deep}
Kaiming He, Xiangyu Zhang, Shaoqing Ren, and Jian Sun.
\newblock Deep residual learning for image recognition.
\newblock In {\em Proceedings of the IEEE conference on computer vision and pattern recognition}, pages 770--778, 2016.

\bibitem{howard2019searching}
Andrew Howard, Mark Sandler, Grace Chu, Liang-Chieh Chen, Bo~Chen, Mingxing Tan, Weijun Wang, Yukun Zhu, Ruoming Pang, Vijay Vasudevan, et~al.
\newblock Searching for mobilenetv3.
\newblock In {\em Proceedings of the IEEE/CVF International Conference on Computer Vision}, pages 1314--1324, 2019.

\bibitem{howard2017mobilenets}
Andrew~G Howard, Menglong Zhu, Bo~Chen, Dmitry Kalenichenko, Weijun Wang, Tobias Weyand, Marco Andreetto, and Hartwig Adam.
\newblock Mobilenets: Efficient convolutional neural networks for mobile vision applications.
\newblock {\em arXiv preprint arXiv:1704.04861}, 2017.

\bibitem{hu2018squeeze}
Jie Hu, Li~Shen, and Gang Sun.
\newblock Squeeze-and-excitation networks.
\newblock In {\em Proceedings of the IEEE conference on computer vision and pattern recognition}, pages 7132--7141, 2018.

\bibitem{huang2017densely}
Gao Huang, Zhuang Liu, Laurens Van Der~Maaten, and Kilian~Q Weinberger.
\newblock Densely connected convolutional networks.
\newblock In {\em Proceedings of the IEEE conference on computer vision and pattern recognition}, pages 4700--4708, 2017.

\bibitem{huang2016deep}
Gao Huang, Yu~Sun, Zhuang Liu, Daniel Sedra, and Kilian~Q Weinberger.
\newblock Deep networks with stochastic depth.
\newblock In {\em European conference on computer vision}, pages 646--661. Springer, 2016.

\bibitem{iandola2016squeezenet}
Forrest~N Iandola, Song Han, Matthew~W Moskewicz, Khalid Ashraf, William~J Dally, and Kurt Keutzer.
\newblock Squeezenet: Alexnet-level accuracy with 50x fewer parameters and< 0.5 mb model size.
\newblock {\em arXiv preprint arXiv:1602.07360}, 2016.

\bibitem{jin2018auto}
Haifeng Jin, Qingquan Song, and Xia Hu.
\newblock Auto-keras: Efficient neural architecture search with network morphism.
\newblock {\em arXiv preprint arXiv:1806.10282}, 5, 2018.

\bibitem{karras2017progressive}
Tero Karras, Timo Aila, Samuli Laine, and Jaakko Lehtinen.
\newblock Progressive growing of gans for improved quality, stability, and variation.
\newblock {\em arXiv preprint arXiv:1710.10196}, 2017.

\bibitem{klein2017learning}
Aaron Klein, Stefan Falkner, Jost~Tobias Springenberg, and Frank Hutter.
\newblock Learning curve prediction with bayesian neural networks.
\newblock In {\em International conference on learning representations}, 2017.

\bibitem{krizhevsky2009learning}
Alex Krizhevsky and Geoffrey Hinton.
\newblock Learning multiple layers of features from tiny images.
\newblock Technical Report~0, University of Toronto, Toronto, Ontario, 2009.

\bibitem{lecun1990optimal}
Yann LeCun, John~S Denker, and Sara~A Solla.
\newblock Optimal brain damage.
\newblock In {\em Advances in neural information processing systems}, pages 598--605, 1990.

\bibitem{liu2018progressive}
Chenxi Liu, Barret Zoph, Maxim Neumann, Jonathon Shlens, Wei Hua, Li-Jia Li, Li~Fei-Fei, Alan Yuille, Jonathan Huang, and Kevin Murphy.
\newblock Progressive neural architecture search.
\newblock In {\em Proceedings of the European conference on computer vision (ECCV)}, pages 19--34, 2018.

\bibitem{liu2018darts}
Hanxiao Liu, Karen Simonyan, and Yiming Yang.
\newblock Darts: Differentiable architecture search.
\newblock {\em arXiv preprint arXiv:1806.09055}, 2018.

\bibitem{liu2019splitting}
Qiang Liu, Lemeng Wu, and Dilin Wang.
\newblock Splitting steepest descent for growing neural architectures.
\newblock {\em arXiv preprint arXiv:1910.02366}, 2019.

\bibitem{liu2017learning}
Zhuang Liu, Jianguo Li, Zhiqiang Shen, Gao Huang, Shoumeng Yan, and Changshui Zhang.
\newblock Learning efficient convolutional networks through network slimming.
\newblock In {\em Proceedings of the IEEE international conference on computer vision}, pages 2736--2744, 2017.

\bibitem{liu2018rethinking}
Zhuang Liu, Mingjie Sun, Tinghui Zhou, Gao Huang, and Trevor Darrell.
\newblock Rethinking the value of network pruning.
\newblock {\em arXiv preprint arXiv:1810.05270}, 2018.

\bibitem{luo2018neural}
Renqian Luo, Fei Tian, Tao Qin, Enhong Chen, and Tie-Yan Liu.
\newblock Neural architecture optimization.
\newblock {\em arXiv preprint arXiv:1808.07233}, 2018.

\bibitem{negrinho2017deeparchitect}
Renato Negrinho and Geoff Gordon.
\newblock Deeparchitect: Automatically designing and training deep architectures.
\newblock {\em arXiv preprint arXiv:1704.08792}, 2017.

\bibitem{nekrasov2019fast}
Vladimir Nekrasov, Hao Chen, Chunhua Shen, and Ian Reid.
\newblock Fast neural architecture search of compact semantic segmentation models via auxiliary cells.
\newblock In {\em Proceedings of the IEEE/CVF Conference on Computer Vision and Pattern Recognition}, pages 9126--9135, 2019.

\bibitem{pham2018efficient}
Hieu Pham, Melody Guan, Barret Zoph, Quoc Le, and Jeff Dean.
\newblock Efficient neural architecture search via parameters sharing.
\newblock In {\em International Conference on Machine Learning}, pages 4095--4104. PMLR, 2018.

\bibitem{real2019regularized}
Esteban Real, Alok Aggarwal, Yanping Huang, and Quoc~V Le.
\newblock Regularized evolution for image classifier architecture search.
\newblock In {\em Proceedings of the aaai conference on artificial intelligence}, volume~33, pages 4780--4789, 2019.

\bibitem{real2017large}
Esteban Real, Sherry Moore, Andrew Selle, Saurabh Saxena, Yutaka~Leon Suematsu, Jie Tan, Quoc~V Le, and Alexey Kurakin.
\newblock Large-scale evolution of image classifiers.
\newblock In {\em International Conference on Machine Learning}, pages 2902--2911. PMLR, 2017.

\bibitem{rusu2016progressive}
Andrei~A Rusu, Neil~C Rabinowitz, Guillaume Desjardins, Hubert Soyer, James Kirkpatrick, Koray Kavukcuoglu, Razvan Pascanu, and Raia Hadsell.
\newblock Progressive neural networks.
\newblock {\em arXiv preprint arXiv:1606.04671}, 2016.

\bibitem{shin2018differentiable}
Richard Shin*, Charles Packer*, and Dawn Song.
\newblock Differentiable neural network architecture search, 2018.

\bibitem{simonyan2014very}
Karen Simonyan and Andrew Zisserman.
\newblock Very deep convolutional networks for large-scale image recognition.
\newblock {\em arXiv preprint arXiv:1409.1556}, 2014.

\bibitem{tan2019mnasnet}
Mingxing Tan, Bo~Chen, Ruoming Pang, Vijay Vasudevan, Mark Sandler, Andrew Howard, and Quoc~V Le.
\newblock Mnasnet: Platform-aware neural architecture search for mobile.
\newblock In {\em Proceedings of the IEEE/CVF Conference on Computer Vision and Pattern Recognition}, pages 2820--2828, 2019.

\bibitem{tan2019efficientnet}
Mingxing Tan and Quoc Le.
\newblock Efficientnet: Rethinking model scaling for convolutional neural networks.
\newblock In {\em International Conference on Machine Learning}, pages 6105--6114. PMLR, 2019.

\bibitem{vaswani}
Ashish Vaswani, Noam Shazeer, Niki Parmar, Jakob Uszkoreit, Llion Jones, Aidan~N Gomez, \L~ukasz Kaiser, and Illia Polosukhin.
\newblock Attention is all you need.
\newblock In I.~Guyon, U.~V. Luxburg, S.~Bengio, H.~Wallach, R.~Fergus, S.~Vishwanathan, and R.~Garnett, editors, {\em Advances in Neural Information Processing Systems}, volume~30. Curran Associates, Inc., 2017.

\bibitem{wang2019energy}
Dilin Wang, Meng Li, Lemeng Wu, Vikas Chandra, and Qiang Liu.
\newblock Energy-aware neural architecture optimization with fast splitting steepest descent.
\newblock {\em arXiv preprint arXiv:1910.03103}, 2019.

\bibitem{wong2018transfer}
Catherine Wong, Neil Houlsby, Yifeng Lu, and Andrea Gesmundo.
\newblock Transfer learning with neural automl.
\newblock {\em arXiv preprint arXiv:1803.02780}, 2018.

\bibitem{wu2020firefly}
Lemeng Wu, Bo~Liu, Peter Stone, and Qiang Liu.
\newblock Firefly neural architecture descent: a general approach for growing neural networks.
\newblock {\em Advances in Neural Information Processing Systems}, 33, 2020.

\bibitem{wu2020steepest}
Lemeng Wu, Mao Ye, Qi~Lei, Jason~D Lee, and Qiang Liu.
\newblock Steepest descent neural architecture optimization: Escaping local optimum with signed neural splitting.
\newblock {\em arXiv preprint arXiv:2003.10392}, 2020.

\bibitem{zhang2018shufflenet}
Xiangyu Zhang, Xinyu Zhou, Mengxiao Lin, and Jian Sun.
\newblock Shufflenet: An extremely efficient convolutional neural network for mobile devices.
\newblock In {\em Proceedings of the IEEE conference on computer vision and pattern recognition}, pages 6848--6856, 2018.

\bibitem{zhong2018practical}
Zhao Zhong, Junjie Yan, Wei Wu, Jing Shao, and Cheng-Lin Liu.
\newblock Practical block-wise neural network architecture generation.
\newblock In {\em Proceedings of the IEEE conference on computer vision and pattern recognition}, pages 2423--2432, 2018.

\bibitem{zoph2018learning}
Barret Zoph, Vijay Vasudevan, Jonathon Shlens, and Quoc~V Le.
\newblock Learning transferable architectures for scalable image recognition.
\newblock In {\em Proceedings of the IEEE conference on computer vision and pattern recognition}, pages 8697--8710, 2018.

\end{thebibliography}
\bibliographystyle{plain}

\clearpage

\appendix 
\section{Gauss-Newton Approximation}
In this section we review the justification for Gauss-Newton approximation. We begin by assuming that the loss function is well-approximated by a least-squares problem in $z$, i.e., for some matrix $A$ and vector $b$,
\begin{align*}
\mathcal L(z) &\approx \frac{1}{2} \| Az - b \|_2^2 \\
&= \frac{1}{2} b^\top b - z^\top A^\top b + \frac{1}{2} z A^\top A z.
\end{align*}
Denote the residual:
\begin{align*}
    r = Az - b
\end{align*}
Note that $\mathcal L = \frac{1}{2} r^\top r$. Denote the gradient and Hessian of the loss:
\begin{align*}
    g &= A^\top r & H &= A^\top A
\end{align*} Consider the change in the loss function when adding a quantity $\Delta z$ to $z$. Denote the change in loss:
\begin{align*}
\Delta \mathcal L(\Delta z) &\equiv \mathcal L(z + \Delta z) - \mathcal L(z) \\
&= \Delta z \cdot g + \frac{1}{2} \Delta z^\top A^\top A \Delta z \\
&= \Delta z \cdot g + \frac{1}{2} \Delta z^\top H \Delta z \\
\end{align*}
In this paper, we write the Gauss-Newton approximation as
\[
H \approx \frac{1}{2 \mathcal L} g g^\top,
\]
\begin{theorem}[General Gauss-Newton Approximation]
If ${\Delta z = \lambda \Delta z^*}$ for some $\lambda \in \mathbb R$ and some $\Delta z^*$ satisfying $A (z + \Delta z^*) = b$, then
\[
\frac{1}{2} \Delta z^\top H \Delta z = \frac{1}{2} \Delta z \frac{gg^\top}{2 \mathcal L} \Delta z
\]
\end{theorem}
\begin{proof}
If ${\Delta z = \lambda \Delta z^*}$ and $A (z + \Delta z^*) = b$, then ${A \Delta z = - \lambda r}$. So
\begin{gather*}
 \frac{1}{2} \Delta z^\top \left( \frac{g g^\top}{2 \mathcal L} \right) \Delta z \\
= \frac{1}{2} \Delta z^\top \left( \frac{A^\top r r^\top A}{r^\top r} \right) \Delta z \\
= \frac{1}{2} \lambda^2 r^\top r
\end{gather*}
Similarly,
\begin{align*}
    \frac{1}{2} \Delta z^\top H  \Delta z &= \frac{1}{2} \Delta z A^\top A \Delta z \\
    &= \frac{1}{2} \lambda^2 r^\top r
\end{align*}
\end{proof}
Therefore, we say the Gauss-Newton approximation is exact in the space spanned by the solutions $\Delta z^*$ to the linear system $A (z + \Delta z^*) = b$.
\begin{theorem} [Rank-1 Gauss-Newton Approximation]
If $H$ is rank-1 and there exists a solution $z^*$ to the linear system $A z^* = b$, then for all $\Delta z$,
\[
\frac{1}{2} \Delta z^\top H \Delta z = \frac{1}{2} \Delta z \frac{gg^\top}{2 \mathcal L} \Delta z
\]
\end{theorem}
\begin{proof}
If $H$ is rank-1, then $A$ consists of a single row $u^\top$ and $b \in \mathbb R$ is a scalar.

Let $\Delta z$ be arbitrary. Then there exists a solution $\Delta z^*$ to the linear system $A(z + \Delta z^*) = b$ and $\lambda \in \mathbb R$ such that $\Delta z = \lambda \Delta z^*$, namely
\begin{gather*}
     \Delta z^* = \frac{b - u^\top z}{u^\top z} \Delta z \\
     \lambda = \frac{u^\top \Delta z}{b - u^\top z}
\end{gather*}
since then
\begin{align*}
    A(z + \Delta z^*) &= u^\top z + u^\top \Delta z^* \\
    &= u^\top z + u^\top \left( \frac{b - u^\top z}{u^\top z} \Delta z \right) \\
    &= b
\end{align*}
Applying the previous theorem yields the result.
\end{proof}
From this it is clear that we can expect the Gauss-Newton approximation to be quite accurate if the true Hessian matrix $H$ is low-rank.

\end{document}